\documentclass{article}
\usepackage{spconf}
\usepackage[cal=boondox,%
calscaled=.97,%
bb=ams,%
bbscaled=.94,%
frak=euler,%
frakscaled=.97,%
scr=esstix]{mathalpha}

\usepackage{amsmath}
\usepackage{amsmath}
\usepackage{amsthm, amssymb, amsfonts, mathtools, cases}
\usepackage{cite}

\usepackage{pgfplots}
\pgfplotsset{compat=newest}
\usepgflibrary{plotmarks}
\usetikzlibrary{
  tikzmark,
  calc,
  arrows,
  decorations.pathmorphing,
  shapes,
  matrix
}

\usepackage{graphicx}
\graphicspath{ {./figs/} }
\DeclareGraphicsExtensions{.png,.pdf}
\usepackage{xcolor}
\usepackage{url}

\definecolor{cadetblue}{rgb}{0.37, 0.62, 0.63}
\definecolor{burntorange}{rgb}{0.8, 0.33, 0.0}
\definecolor{americanrose}{rgb}{1.0, 0.01, 0.24}
\definecolor{applegreen}{rgb}{0.55, 0.71, 0.0}
\definecolor{darkmagenta}{rgb}{0.55, 0.0, 0.55}
\definecolor{peru}{rgb}{0.80, 0.52, 0.25}
\definecolor{navy}{rgb}{0.0, 0.0, 0.5}
\definecolor{maroon}{rgb}{0.5, 0.0, 0.0}
\definecolor{gold}{rgb}{1.0, 0.84, 0.0}
\definecolor{crimson}{rgb}{0.86, 0.08, 0.24}

\usepackage{newfloat}

\usepackage{subcaption}

\usepackage[inline]{enumitem}
\usepackage{etoolbox}

\usepackage{algorithm, algpseudocode}
\makeatletter
\newcommand{\algmargin}{\the\ALG@thistlm}
\makeatother
\newlength{\whilewidth}
\algdef{SE}[parWHILE]{parWhile}{EndparWhile}[1]
  {\parbox[t]{\dimexpr\linewidth-\algmargin}{%
     \hangindent\whilewidth\strut\algorithmicwhile\ #1\ \algorithmicdo\strut}}{\algorithmicend\
   \algorithmicwhile}%
\algnewcommand{\parState}[1]{\State\parbox[t]{\dimexpr\linewidth-\algmargin}{\strut #1\strut}}

\newcommand{\IntegerP}{\mathbb{N}}
\newcommand{\IntegerPP}{\mathbb{N}_*}
\newcommand{\Real}{\mathbb{R}}

\newcommand\given{{\mathbin{}\mid\mathbin{}}}
\newcommand\vect[1]{\mathbf{#1}}
\newcommand\vectgr[1]{\boldsymbol{#1}}

\providecommand\given{} 
\newcommand\SetSymbol[1][]{
  \nonscript\,#1\vert \allowbreak \nonscript\,\mathopen{}}
\DeclarePairedDelimiterX\Set[1]{\lbrace}{\rbrace}%
{ \renewcommand\given{\SetSymbol[\delimsize]} #1 }
\DeclarePairedDelimiterX\innerp[2]{\langle}{\rangle}{#1
  \mathop{}\delimsize\vert\mathop{} #2}
\DeclarePairedDelimiterX\norm[1]\lVert\rVert{\ifblank{#1}{\:\cdot\:}{#1}}

\DeclareMathOperator{\Fix}{Fix}

\usepackage{thmtools}
\usepackage[unq]{unique}

\declaretheoremstyle[%
headfont=\normalfont\bfseries,
notefont=\mdseries,
notebraces={(}{)},
bodyfont=\normalfont,
postheadspace=1ex
]{mystyle}

\declaretheorem[style=mystyle,
                name=Theorem,
                refname={theorem,theorems},
                Refname={Theorem,Theorems}
]{thm}

\declaretheorem[style=mystyle,
                name=Proposition,
                refname={proposition,propositions},
                Refname={proposition,propositions}
]{prop}

\newlist{thmlist}{enumerate}{1}
\setlist[thmlist]{label=\textbf{(\roman{*})}, ref=\thethm(\roman{*}), noitemsep}

\newlist{lemlist}{enumerate}{1}
\setlist[lemlist]{label=\textbf{(\roman{*})}, ref=\thelemma(\roman{*}), noitemsep}

\newlist{exlist}{enumerate}{1}
\setlist[exlist]{label=\textbf{(\roman{*})}, ref=\theexample(\roman{*}), noitemsep}

\newlist{factlist}{enumerate}{1}
\setlist[factlist]{label=\textbf{(\roman{*})}, ref=\thefact(\roman{*}), noitemsep}

\newlist{proplist}{enumerate}{1}
\setlist[proplist]{label=\textbf{(\roman{*})}, ref=\theprop(\roman{*}), noitemsep}

\newlist{asslist}{enumerate}{1}
\setlist[asslist]{label=\textbf{(\roman{*})},
  ref=\theassumption(\roman{*}), noitemsep}

\newlist{deflist}{enumerate}{1}
\setlist[deflist]{label=\textbf{(\roman{*})}, ref=\thedefinition(\roman{*}), noitemsep}

\newlist{algolist}{enumerate}{1}
\setlist[algolist]{label=\textbf{(\roman{*})}, ref=\thealgo(\roman{*}), noitemsep}

\newlist{claimlist}{enumerate}{1}     
\setlist[claimlist]{label=\textbf{(\roman{*})}, ref=\theclaim(\roman{*}), noitemsep}

\newlist{applist}{enumerate}{1}
\setlist[applist]{label=\textbf{(\roman{*})}, ref=\thesection(\roman{*}), noitemsep}

\newlist{MyEnumSec}{enumerate}{1}
\setlist[MyEnumSec]{label=\textbf{\thesection(\roman{*})},
  ref=Item~\thesection(\roman{*}), noitemsep}

\newlist{MyEnumSubSec}{enumerate}{1}
\setlist[MyEnumSubSec]{label=\textbf{\thesubsection(\roman{*})},
  ref=Item~\thesubsection(\roman{*}), noitemsep, wide = 0pt, leftmargin = *}

\usepackage[capitalize]{cleveref}

\crefname{thm}{Theorem}{Theorems}
\crefname{prop}{Proposition}{Propositions}
\crefname{assumption}{Assumption}{Assumptions}
\crefname{lemma}{Lemma}{Lemmata}
\crefname{definition}{Definition}{Definitions}
\crefname{example}{Example}{Examples}
\crefname{algo}{Algorithm}{Algorithms}
\crefname{fact}{Fact}{Facts}
\crefname{claim}{Claim}{Claims}
\crefname{appendix}{Appendix}{Appendices}
\crefname{coroll}{Corollary}{Corollaries}
\crefname{figure}{Figure}{Figures}
\crefname{section}{Section}{Sections}

\crefname{thmlisti}{Theorem}{Theorems}
\crefname{lemlisti}{Lemma}{Lemmata}
\crefname{proplisti}{Proposition}{Propositions}
\crefname{asslisti}{Assumption}{Assumptions}
\crefname{deflisti}{Definition}{Definitions}
\crefname{exlisti}{Example}{Examples}
\crefname{algolisti}{Algorithm}{Algorithms}
\crefname{factlisti}{Fact}{Facts}
\crefname{claimlisti}{Claim}{Claims}
\crefname{applisti}{Appendix}{Appendices}
\crefname{MyEnumSeci}{}{}
\crefname{MyEnumSubSeci}{}{}

\makeatletter

\newcommand*{\ie}{%
  \@ifnextchar{,}%
  {\textit{i.e.}}%
  {\textit{i.e.,}\@\xspace}%
}
\newcommand*{\eg}{%
  \@ifnextchar{,}%
  {\textit{e.g.}}%
  {\textit{e.g.,}\@\xspace}%
}
\newcommand*{\etc}{%
  \@ifnextchar{.}%
  {\textit{etc}}%
  {\textit{etc.}\@\xspace}%
}
\newcommand*{\etal}{%
  \@ifnextchar{.}%
  {\textit{et al}}%
  {\textit{et al.}\@\xspace}%
}
\newcommand*{\cf}{%
  \@ifnextchar{.}%
  {\textit{cf}}%
  {\textit{cf.}\@\xspace}%
}
\newcommand*{\aka}{%
  \@ifnextchar{,}%
  {\textit{a.k.a.}}%
  {\textit{a.k.a.}\@\xspace}%
}
\makeatother

\allowdisplaybreaks

\title{Online And Lightweight Kernel-Based Approximate Policy Iteration for Dynamic
  p-Norm Linear Adaptive Filtering\vspace{-25pt}%
}

\name{}

\address{%
  \begin{minipage}{.7\textwidth}
    \begin{center}
      \textit{Yuki Akiyama\qquad Minh Vu\qquad Konstantinos Slavakis}\\[1ex] \small
      Tokyo Institute of Technology, Japan\\
      Department of Information and Communications Engineering\\
      Emails: \texttt{\{akiyama.y.am, vu.d.aa, slavakis.k.aa\}@m.titech.ac.jp}
    \end{center}
  \end{minipage}
  \vspace{-15pt}
}

\begin{document}
\ninept

\maketitle
\begin{abstract}
  This paper introduces a solution to the problem of selecting dynamically (online) the
  ``optimal'' p-norm to combat outliers in linear adaptive filtering without any knowledge on
  the probability density function of the outliers. The proposed online and data-driven
  framework is built on kernel-based reinforcement learning (KBRL). To this end, novel Bellman
  mappings on reproducing kernel Hilbert spaces (RKHSs) are introduced. These mappings do not
  require any knowledge on transition probabilities of Markov decision processes, and are
  nonexpansive with respect to the underlying Hilbertian norm. The fixed-point sets of the
  proposed Bellman mappings are utilized to build an approximate policy-iteration (API)
  framework for the problem at hand. To address the ``curse of dimensionality'' in RKHSs,
  random Fourier features are utilized to bound the computational complexity of the
  API. Numerical tests on synthetic data for several outlier scenarios demonstrate the superior
  performance of the proposed API framework over several non-RL and KBRL schemes.
\end{abstract}

\section{Introduction}\label{sec:intro}

The least-squares (LS) error/loss (between an observed value and its predicted one) plays a
pivotal role in signal processing, e.g., adaptive filtering~\cite{sayed2011adaptive}, and
machine learning~\cite{Theodoridis.Book:ML}. For example, the least-mean squares (LMS) and
recursive (R)LS~\cite{sayed2011adaptive} are two celebrated algorithms in adaptive filtering
and stochastic approximation based on the LS-error criterion. Notwithstanding, LS methods are
notoriously sensitive to the presence of outliers within data~\cite{rousseeuw1987}, where
outliers are defined as (sparsely) contaminating data that do not adhere to a nominal data
generation model, and are often modeled as random variables (RVs) with non-Gaussian heavy
tailed distributions, e.g., $\alpha$-stable ones~\cite{shao1993signal}. To combat outliers,
several non-LS criteria, such as least mean p-power (LMP)~\cite{pei1994p-power,
  xiao1999adaptive, Kuruoglu:02, vazquez2012, chen2015smoothed, slavakis2021outlier} and
maximum correntropy (MC)~\cite{Singh.MCC:09}, have been studied. This work focuses on the LMP
criterion, owing to the well-documented robustness of LMP against outliers~\cite{Gentile:03},
while results on MC will be reported elsewhere.

This study is built on the classical data-generation model
$y_n = \vectgr{\theta}_*^{\intercal} \vect{x}_n + o_n$, where $n\in\IntegerP$ denotes discrete
time ($\IntegerP$ is the set of all non-negative integers), $\vectgr{\theta}_*$ is the
$L\times 1$ vector whose entries are the system parameters that need to be identified, $o_n$ is
the RV which models outliers/noise, $(\vect{x}_n, y_n)$ stands for the input-output pair of
available data, where $\vect{x}_n$ is an $L\times 1$ vector and $y_n$ is real-valued, and
$\intercal$ denotes vector/matrix transposition. For an arbitrarily fixed $\vectgr{\theta}_0$,
the LMP algorithm~\cite{pei1994p-power} generates estimates
$(\vectgr{\theta}_n)_{n\in\mathbb{N}}$ of $\vectgr{\theta}_*$ according to the following
recursion:
\begin{equation}
  \vectgr{\theta}_{n+1} \coloneqq \vectgr{\theta}_n + \rho p \lvert e_n\rvert^{p-2}e_n
  \vect{x}_n \,, \label{LMP}
\end{equation}
where $e_n \coloneqq y_n - \vect{x}_n^{\intercal} \vectgr{\theta}_n$, $\rho$ is the learning
rate (step size), and $p$ is a \textit{fixed}\/ user-defined real-valued number within the
interval $[1, 2]$ to ensure that the $p$-norm loss
$\lvert y_n - \vect{x}_n^{\intercal} \vectgr{\theta} \rvert^p$ is a convex function of
$\vectgr{\theta}$~\cite{pei1994p-power}. Notice that if $p = 1$ and $2$, then \eqref{LMP} boils
down to the classical sign-LMS and LMS, respectively~\cite{sayed2011adaptive}.

Intuition suggests that the choice of $p$ should be based on the probability density
function (PDF) of the RV $o_n$. For example, if $o_n$ obeys a Gaussian PDF, then $p=2$ should
be chosen (recall the maximum-likelihood criterion). To enhance robustness against outliers,
combination of adaptive filters with different forgetting factors, but with the same fixed
$p$-norm, have been also introduced~\cite{vazquez2012}. Nevertheless, it seems that an
\textit{online}\/ and \textit{data-driven}\/ solution to the problem of \textit{dynamically}\/
selecting $p$, \textit{without}\/ any prior knowledge on the PDF of $o_n$, is yet to be found.

This work offers a solution to the aforementioned open problem via reinforcement learning
(RL)~\cite{bertsekas2019reinforcement}; a machine-learning paradigm where an ``agent''
interacts with the surrounding environment to identify iteratively the policy which minimizes
the cost of its ``actions.'' More specifically, the well-known policy-iteration (PI)
framework~\cite{bertsekas2019reinforcement} of RL is adopted, because of its well-documented
merits (e.g., \cite{ormoneit2002kernel, Ormoneit:Autom:02, xu2007klspi}) over the alternative
RL frameworks of temporal-difference (TD) and Q-learning~\cite{bertsekas2019reinforcement},
especially for continuous and high-dimensional state spaces. PI comprises two stages at every
iteration $n$: \textit{policy evaluation} and \textit{policy improvement}. At policy
evaluation, the current policy is evaluated by a
$Q$-function~\cite{bertsekas2019reinforcement}, which represents, loosely speaking, the
long-term cost that the agent would suffer had the current policy been chosen to determine the
next state, whereas at the policy-improvement stage, the agent uses the $Q$-function value to
update the policy. The underlying state space is considered to be continuous, due to the nature
of $(\mathbf{x}_n, y_n)$, while the action space is considered to be discrete: an action is a
value of $p$ taken from a finite grid of the interval $[1,2]$.

Deep neural networks offer approximating spaces for Q-functions, e.g., \cite{DDQN}, but they
may require processing of batch data (even re-training) during online-mode operation, since
they may face test data generated by PDFs different from those of the training ones (dynamic
environments). Such batch processing inflicts large computational times and complexity,
discouraging the application of deep neural networks to online modes of operation where a small
computational footprint is desired. To meet such computational complexity requirements,
this study builds an approximate (A)PI framework for \textit{online}\/ RL along the lines of
kernel-based (KB)RL~\cite{ormoneit2002kernel, Ormoneit:Autom:02, xu2007klspi, Bae:MLSP:11,
  Barreto:NIPS:11, Barreto:NIPS:12, Kveton_Theocharous_2013, OnlineBRloss:16, RegularizedPI:16,
  Kveton_Theocharous_2021, Wang_Principe:SPM:21}.

Central to the proposed API is the construction of novel Bellman mappings~\cite{bellman2003dp,
  bertsekas2019reinforcement}. The proposed Bellman mappings are defined on a reproducing
kernel Hilbert space (RKHS) $\mathcal{H}$~\cite{aronszajn1950, scholkopf2002learning}, which
serves as the approximating space for the $Q$-functions. Unlike the classical Bellman
operators, where information on transition probabilities in a Markov decision process is
needed~\cite{bertsekas2019reinforcement}, the proposed Bellman mappings make no use of such
information, and need neither training/offline data nor past policies, but sample and average
the sample space on the fly, at each iteration $n$, to perform \textit{exploration}\/ of the
surrounding environment. This suits the current adaptive-filtering setting, where the presence
of outliers, with a possibly time-varying PDF, may render the information obtained offline or
from past policies outdated. As such, the proposed Bellman mappings fall closer
to~\cite{Ormoneit:Autom:02} than to studies which use training data collected beforehand
(offline), e.g., \cite{lagoudakis2003lspi, xu2007klspi, panaganti2022robust}. 

Further, in contrast to the prevailing route in KBRL~\cite{ormoneit2002kernel,
  Ormoneit:Autom:02, Barreto:NIPS:11, Barreto:NIPS:12, Kveton_Theocharous_2013,
  OnlineBRloss:16, RegularizedPI:16, Kveton_Theocharous_2021}, which views Bellman mappings as
contractions in $\mathcal{L}_{\infty}$-norm Banach spaces (by definition, no inner product
available), this study introduces nonexpansive~\cite{HB.PLC.book} Bellman operators on
$\mathcal{H}$ to capitalize on the reproducing property of the inner product of
$\mathcal{H}$~\cite{aronszajn1950, scholkopf2002learning}, and to open the door to powerful
Hilbertian tools~\cite{HB.PLC.book}. A byproduct of this path is the additional flexibility
offered to the user by the fact that the fixed-point set of a nonexpansive mapping is
non-singleton in general, as opposed to the case of a contraction mapping which is known to
have a unique fixed point. Supersets of those fixed-point sets are designed to build the
proposed API framework.


It is worth stressing here that the proposed API framework, together with its complementary
study~\cite{Minh:ICASSP23}, appear to be the first attempts to apply RL arguments to robust
adaptive filtering. In contrast to \cite{Minh:ICASSP23}, where the state space is the
high-dimensional $\Real^{2L+1}$ ($\Real$ is the set of all real numbers), this study confines
the state space to the low-dimensional $\Real^4$. Moreover, this study constructs potentially
infinite-dimensional hyperplanes as supersets of the fixed-point sets of a proposed Bellman
mappings, as opposed to \cite{Minh:ICASSP23} where finite-dimensional affine sets are
designed. To address the ``curse of dimensionality,'' which arises naturally in online learning
in RKHSs ($\mathcal{H}$ may be infinite dimensional), the proposed framework uses random
Fourier features (RFF)~\cite{RFF, Konidaris:RL_Fourier:11} to bound the computational
complexity of the proposed API, while the approximate-linear-dependency
criterion~\cite{engel2004krls}, which does not ensure a bounded computational complexity, is
used in~\cite{Minh:ICASSP23}. Finally, to robustify the proposed scheme, experience
replay~\cite{ExperienceReplay} is applied, whereas \cite{Minh:ICASSP23} employs
rollout~\cite{bertsekas2019reinforcement}.

Numerical tests on synthetic data showcase the promising performance of the the advocated
framework, which outperforms several RL and non-RL schemes. Due to space limitations, long
proofs, the convergence analysis of the proposed framework, and further numerical tests will be
reported in the journal version of this paper.

\section{Nonexpansive Bellman Mappings on RKHS\MakeLowercase{s}}\label{sec:nonexp.Bellman}

\subsection{State-Action Space}\label{sec:state.action.space}

Following the setting of \eqref{LMP}, the state space $\mathfrak{S}$ is assumed to be
continuous. In contrast to~\cite{Minh:ICASSP23}, where the state space is the high dimensional
$\Real^{2L+1}$, this study considers the case where $\mathfrak{S} \coloneqq \Real^4$, with the
dimension of $\mathfrak{S}$ rendered independent of $L$. Due to
the streaming nature of $(\vect{x}_n, y_n)_{n\in\IntegerP}$, state vectors
$(\vect{s}_n \coloneqq [ s_1^{(n)}, s_2^{(n)}, s_3^{(n)}, s_4^{(n)}
]^{\intercal})_{n\in\IntegerP}$ are defined inductively by the following heuristic rules:
\begin{subequations}\label{def.states}
  \begin{align}
    s_1^{(n)}
    & \coloneqq \log_{10} \lvert y_n-\vectgr{\theta}_n^\intercal\vect{x}_n \rvert
      \,, \label{state1}\\
    s_2^{(n)}
    & \coloneqq \tfrac{1}{M_{\textnormal{av}}} \sum\nolimits_{k=1}^{M_{\textnormal{av}}} \log_{10}
      \frac{\lvert y_{n-k} - \vectgr{\theta}_n^\intercal
      \vect{x}_{n-k}\rvert}{\norm{\vect{x}_{n-k}}_2} \,, \label{state2}\\
    s_3^{(n)}
    & \coloneqq \log_{10} \norm{\vect{x}_n}_2 \,, \label{state3}\\
    s_4^{(n)}
    & \coloneqq \varpi s_4^{(n-1)} + (1 - \varpi) \log_{10} \frac{ \norm{\vectgr{\theta}_{n} -
      \vectgr{\theta}_{n-1}}_2 }{\rho} \,, \label{state4}
  \end{align}
\end{subequations}
where $M_{\textnormal{av}}\in \IntegerPP$, $\varpi \in (0,1)$ are user-defined parameters,
while $\rho$ comes from \eqref{LMP}. The classical \textit{prior loss}\/ of adaptive
filtering~\cite{sayed2011adaptive} is used in \eqref{state1}, an $M_{\textnormal{av}}$-length
sliding-window sampling average of the \textit{posterior loss}~\cite{sayed2011adaptive} is
provided in \eqref{state2}, normalized by the norm of the input signal to remove as much as
possible its effect on the error, the instantaneous norm of the input signal in \eqref{state3},
and a smoothing auto-regressive process in \eqref{state4} to monitor the consecutive
displacement of the estimates $(\vectgr{\theta}_n)_{n\in\IntegerP}$. The reason for including
$\rho$ in \eqref{state4} is to remove $\rho$'s effect from $s_4^{(n)}$. Owing to \eqref{LMP},
the initial value $s_4^{(0)}$ in \eqref{state4} is set equal to
$\log_{10}[ (1/\rho) \norm{\vectgr{\theta}_1 - \vectgr{\theta}_0}_2] = \log_{10}{p_0} + (p_0-1)
s_1^{(0)} + s_3^{(0)}$. The $\log_{10}(\cdot)$ function is employed to decrease the dynamic
range of the positive values in \eqref{def.states}.

The action space $\mathfrak{A}$ is defined as any finite grid of the interval $[1,2]$, so that
an action $a\in \mathfrak{A}$ becomes any value of $p$ taken from that finite grid. The
state-action space is defined as $\mathfrak{Z} \coloneqq \mathfrak{S}\times \mathfrak{A}$, and
its element is denoted as $\vect{z} = (\vect{s}, a)$.

Along the lines of the general notation in~\cite{bertsekas2019reinforcement}, consider now the
set of all mappings
$\mathcal{M} \coloneqq \Set{ \mu(\cdot) \given \mu(\cdot): \mathfrak{S} \to \mathfrak{A}:
  \vect{s} \mapsto \mu(\vect{s})}$. In other words, given a $\mu\in \mathcal{M}$,
$\mu(\vect{s})$ denotes the action that the ``system'' may take at state $\vect{s}$ to ``move
to'' the state $\vect{s}^{\prime}\in \mathfrak{S}$. The one-step loss for this transition is
denoted by $g: \mathfrak{Z} \to \Real: (\vect{s}, a) \mapsto g(\vect{s}, a)$. The set $\Pi$ of
policies is defined as
$\Pi \coloneqq \mathcal{M}^{\IntegerP} \coloneqq \Set{ (\mu_0, \mu_1, \ldots, \mu_n, \ldots)
  \given \mu_n\in \mathcal{M}, n\in \IntegerP}$. A policy will be denoted by $\pi\in
\Pi$. Given $\mu\in \mathcal{M}$, the stationary policy $\pi_{\mu} \in \Pi$ is defined as
$\pi_{\mu} \coloneqq ( \mu, \mu, \ldots, \mu, \ldots)$. It is customary for $\mu$ to denote
also the stationary policy $\pi_{\mu}$. Function
$Q: \mathfrak{Z}\to \Real: (\vect{s}, a) \mapsto Q(\vect{s}, a)$ quantifies the long-term cost
that the agent would suffer had the action $a$ been used to determine the next state of
$\vect{s}$.

\subsection{Novel Bellman Mappings}\label{sec:new.Bellman.maps}

Central to dynamic programming and RL~\cite{bertsekas2019reinforcement} is the concept of
Bellman mappings, which operate on $Q$-functions. Typical definitions, e.g.,
\cite{Bellemare:16}, are as follows: $\forall (\vect{s}, a)\in \mathfrak{Z}$,
\begin{subequations}\label{Bellman.maps.standard}
  \begin{align}
    (T_{\mu}^{\diamond} Q)(\vect{s}, a)
    & \coloneqq g( \vect{s}, a ) + \alpha \mathbb{E}_{\vect{s}^{\prime} \given (\vect{s}, a)}
      \{ Q(\vect{s}^{\prime}, \mu(\vect{s}^{\prime})) \}\,, \label{Bellman.standard.mu} \\
    (T^{\diamond} Q)(\vect{s}, a)
    & \coloneqq g( \vect{s}, a ) + \alpha \mathbb{E}_{\vect{s}^{\prime} \given (\vect{s}, a)}
      \{ \min_{a^{\prime}\in \mathfrak{A} }Q(\vect{s}^{\prime}, a^{\prime})
      \}\,, \label{Bellman.standard}
  \end{align}
\end{subequations}
where $\mathbb{E}_{\vect{s}^{\prime} \given (\vect{s}, a)}\{\cdot\}$ stands for the conditional
expectation operator with respect to $\vect{s}^{\prime}$ conditioned on $(\vect{s}, a)$, and
$\alpha$ is the discount factor with typical values in $[0,1)$. In the case where $Q$ is
considered an element of the Banach space of all (essentially) bounded
functions~\cite{Bartle.book:95}, equipped with the $\mathcal{L}_{\infty}$-norm
$\norm{}_{\infty}$, then it can be shown that the mappings in \eqref{Bellman.maps.standard} are
contractions~\cite{bertsekas2019reinforcement}, and according to the Banach-Picard
theorem~\cite{HB.PLC.book}, they possess \textit{unique}\/ fixed points
$Q^{\diamond}_{\mu}, Q^{\diamond}$, i.e., points which solve the Bellman equations
$T_{\mu}^{\diamond} Q^{\diamond}_{\mu} = Q^{\diamond}_{\mu}$ and
$T^{\diamond} Q^{\diamond} = Q^{\diamond}$, and which characterize ``optimal'' long-term
losses~\cite{bertsekas2019reinforcement}.

Nevertheless, in most cases of practical interest, there is not sufficient information on the
conditional probability distribution to compute the expectation operator in
\eqref{Bellman.maps.standard}. To this end, this study proposes approximations of the Bellman
mappings in \eqref{Bellman.maps.standard} by assuming that losses $g$ and $Q$ belong to an RKHS
$\mathcal{H}$, i.e., a Hilbert space with inner product $\innerp{\cdot}{\cdot}_{\mathcal{H}}$,
norm $\norm{\cdot}_{\mathcal{H}} \coloneqq \innerp{\cdot}{\cdot}_{\mathcal{H}}^{1/2}$, and a
reproducing kernel $\kappa(\cdot, \cdot): \mathfrak{Z}\times \mathfrak{Z} \to \Real$, such that
$\kappa(\vect{z}, \cdot)\in \mathcal{H}$, $\forall \vect{z}\in \mathfrak{Z}$, and the reproducing
property holds true: $Q(\vect{z}) = \innerp{Q}{\kappa(\vect{z}, \cdot)}_{\mathcal{H}}$,
$\forall Q\in \mathcal{H}$, $\forall \vect{z}\in \mathfrak{Z}$. Space $\mathcal{H}$ may be
infinite dimensional; e.g., $\kappa(\cdot, \cdot)$ is a Gaussian kernel~\cite{aronszajn1950,
  scholkopf2002learning}. For compact notations, let $\varphi(\vect{z}) \coloneqq
\kappa(\vect{z},\cdot)$, and $Q^{\intercal} Q^{\prime} \coloneqq
\innerp{Q}{Q^{\prime}}_{\mathcal{H}}$.

Hereafter, losses $g, Q$ are assumed to belong to $\mathcal{H}$. The proposed Bellman mappings
$T_{\mu}: \mathcal{H} \to \mathcal{H}: Q\mapsto T_{\mu}Q$ and
$T: \mathcal{H} \to \mathcal{H}: Q\mapsto TQ$ are defined as:
\begin{subequations}\label{Bellman.maps.new}
  \begin{align}
    T_{\mu} Q & \coloneqq g + \alpha \sum\nolimits_{j=1}^{N_{\textnormal{av}}}
                Q (\vect{s}^{\textnormal{av}}_j, \mu(\vect{s}^{\textnormal{av}}_j) ) \cdot \psi_j
                \,, \label{Bellman.new.mu}\\
    TQ & \coloneqq g + \alpha \sum\nolimits_{j=1}^{N_{\textnormal{av}}} \inf\nolimits_{a_j \in
         \mathfrak{A}} Q (\vect{s}^{\textnormal{av}}_j, a_j) )\cdot \psi_j
         \,, \label{Bellman.new}
  \end{align}
\end{subequations}
where $\{\psi_j\}_{j=1}^{N_{\textnormal{av}}}$ are vectors in $\mathcal{H}$, for a
user-defined positive integer $N_{\textnormal{av}}$, and
$\{ \vect{s}^{\textnormal{av}}_j \}_{j=1}^{N_{\textnormal{av}}}$ are state vectors chosen by
the user for the summations in \eqref{Bellman.maps.new} to approximate the conditional
expectations in \eqref{Bellman.maps.standard}. See for example \cite{Minh:ICASSP23}, where
$\{ \vect{s}^{\textnormal{av}}_j \}_{j=1}^{N_{\textnormal{av}}}$ are drawn from a Gaussian
distribution centered at a state of interest (the current state $\vect{s}_n$ in
\cref{sec:tests}). For notational convenience, let
$\vectgr{\Psi} \coloneqq [\psi_1, \ldots, \psi_{N_{\textnormal{av}}} ]$, and its
$N_{\textnormal{av}} \times N_{\textnormal{av}}$ kernel matrix
$\vect{K}_{\Psi} \coloneqq \vectgr{\Psi}^{\intercal} \vectgr{\Psi}$ whose $(j, j^{\prime})$
entry is equal to $\innerp{\psi_j}{\psi_{j^{\prime}}}_{\mathcal{H}}$. Moreover, let
$\vectgr{\Phi}_{\mu}^{\textnormal{av}} \coloneqq [\varphi^{\textnormal{av}}_{\mu,1}, \ldots,
\varphi^{\textnormal{av}}_{\mu,N_{\textnormal{av}}}]$, where
$\varphi^{\textnormal{av}}_{\mu, j} \coloneqq \varphi( \vect{s}_j^{\textnormal{av}},
\mu(\vect{s}_j^{\textnormal{av}}) )$, with kernel matrix
$\vect{K}^{\textnormal{av}}_{\mu} \coloneqq
{\vectgr{\Phi}}^{\textnormal{av}}_{\mu}{}^{\intercal} {\vectgr{\Phi}}^{\textnormal{av}}_{\mu}$.

\begin{thm}\label{thm:nonexp} Let $\psi_j(\vect{z})\geq 0$, $\forall \vect{z}\in
  \mathfrak{Z}$, $\forall j\in \{1, \ldots, N_{\textnormal{av}}\}$. If
  $\alpha \leq \norm{\vect{K}_{\Psi}}^{-1/2} (\sup_{\mu\in\mathcal{M}} \norm{
    \vect{K}^{\textnormal{av}}_{\mu}} )^{-1/2}$, then $\forall \mu\ \in \mathcal{M}$, the
  mapping $T_{\mu}$ in \eqref{Bellman.new.mu} is affine nonexpansive and $T$ in
  \eqref{Bellman.new} is nonexpansive within the Hilbert space
  $( \mathcal{H}, \innerp{\cdot}{\cdot}_{\mathcal{H}} )$. Norms
  $\norm{\vect{K}_{\Psi}}, \norm{\vect{K}^{\textnormal{av}}_{\mu}}$ are the spectral norms
  of $\vect{K}_{\Psi}, \vect{K}^{\textnormal{av}}_{\mu}$.
\end{thm}

Nonexpansivity for $T$ in a (Euclidean) Hilbert space
$( \mathcal{H}, \innerp{\cdot}{\cdot}_{\mathcal{H}} )$ means
$\norm{ TQ - TQ^{\prime}}_{\mathcal{H}} \leq \norm{ Q - Q^{\prime}}_{\mathcal{H}}$,
$\forall Q, Q^{\prime}\in \mathcal{H}$~\cite{HB.PLC.book}. Moreover,
$T_{\mu}: \mathcal{H} \to \mathcal{H}$ is affine iff
$T_{\mu}( \lambda Q + (1-\lambda)Q^{\prime}) = \lambda TQ + (1 - \lambda)TQ^{\prime}$,
$\forall Q, Q^{\prime}\in \mathcal{H}$, $\forall\lambda \in \Real$.

Mappings \eqref{Bellman.maps.new} share similarities with those in~\cite{ormoneit2002kernel,
  Ormoneit:Autom:02, Barreto:NIPS:11, Barreto:NIPS:12, Kveton_Theocharous_2013,
  Kveton_Theocharous_2021}. However, in~\cite{ormoneit2002kernel, Ormoneit:Autom:02,
  Barreto:NIPS:11, Barreto:NIPS:12, Kveton_Theocharous_2013, Kveton_Theocharous_2021} as well
as in the classical context of \eqref{Bellman.maps.standard}, Bellman mappings are viewed as
contractions on the Banach space of (essentially) bounded functions with the
$\mathcal{L}_{\infty}$-norm~\cite{bertsekas2019reinforcement}, while no discussion on RKHSs is
reported. Recall that, by definition, Banach spaces are not equipped with inner products. On
the other hand, \cref{thm:nonexp} opens the door not only to the rich toolbox of nonexpansive
mappings in Hilbert spaces~\cite{HB.PLC.book}, but also to the reproducing property of the
inner product in RKHSs~\cite{aronszajn1950, scholkopf2002learning}.

\section{Approximate Policy Iteration}\label{sec:algo}

\begin{algorithm}[t]
  \begin{algorithmic}[1]
    \renewcommand{\algorithmicindent}{1em}

    \State{Arbitrarily initialize $Q_0$, $\mu_0\in\mathcal{M}$, and $\vectgr{\theta}_0\in
      \Real^L$.}

    \While{$n \in \mathbb{N}$}\label{line:iter}

    \State{Data $(\vect{x}_n, y_n)$ become available. Let $\vect{s}_n$ as in
      \eqref{def.states}.}

    \State{\textbf{Policy improvement:} Update $a_n \coloneqq \mu_n(\vect{s}_n)$ by
      \eqref{policy.improvement}.}\label{algo:policy.improvement}

    \State{Update $\vectgr{\theta}_{n+1}$ by \eqref{LMP}, where $p \coloneqq
      a_n = \mu_n(\vect{s}_n)$.}

    \State{Define $\{\vect{s}_j^{\textnormal{av}}[n]\}_{j=1}^{N_{\textnormal{av}}[n]}$ (see
      \cref{sec:algo}).}

    \State{Run experience replay on $Q_n$ (see \cref{sec:algo}).}

    \State{\textbf{Policy evaluation:} Update $Q_{n+1}$ by \eqref{Q.update}.}

    \State{Increase $n$ by one, and go to Line \ref{line:iter}.}

    \EndWhile
  \end{algorithmic}

  \caption{Approximate policy iteration for LMP.}\label{algo}

\end{algorithm}

With the Bellman mappings \eqref{Bellman.maps.new} serving as approximations of the classical
ones \eqref{Bellman.maps.standard}, \cref{algo} offers an \textit{approximate}\/ policy
iteration (API) framework. The framework operates sequentially, with its iteration index $n$
coinciding with the time index of the streaming data $(\vect{x}_n, y_n)_{n\in \IntegerP}$ of
\eqref{LMP}. To this end, the arguments of \cref{sec:nonexp.Bellman} are adapted to include
hereafter the extra time dimension $n$, which will be indicated by the super-/sub-scripts
$[n]$, $(n)$ or $n$ in notations.

\cref{algo} follows the standard path of PI~\cite{bertsekas2019reinforcement}. Policy
improvement is performed in \cref{algo:policy.improvement} of \cref{algo} according to the
standard greedy rule of~\cite{bertsekas2019reinforcement}
\begin{align}
  \mu_n(\vect{s}_n) \coloneqq \arg\min\nolimits_{a\in \mathfrak{A}} Q_n( \vect{s}_n, a)
  \,. \label{policy.improvement}
\end{align}
A different way for policy improvement via rollout can be found in~\cite{Minh:ICASSP23}.

The following proposition constructs a superset for the fixed-point set of
$\Fix T_{\mu_n}^{(n)}$. The superset $\mathscr{H}_n$ is a potentially infinite-dimensional
hyperplane, in contrast to the superset in~\cite{Minh:ICASSP23} which is a finite-dimensional
affine set.

\begin{prop}\label{prop}
  The fixed-point set
  $\Fix T_{\mu_n}^{(n)} \coloneqq \{ Q\in \mathcal{H} \given T_{\mu_n}^{(n)}Q = Q \}$ is a
  subset of the hyperplane
  $\mathscr{H}_n \coloneqq \{ Q\in \mathcal{H} \given g_n(\vect{z}_n) =
  \innerp{Q}{h_n}_{\mathcal{H}} \}$, where
  $h_n \coloneqq \varphi(\vect{z}_n) - \alpha ({1}/{N_{\textnormal{av}}[n]})
  \sum\nolimits_{j=1}^{N_{\textnormal{av}}[n]} \varphi^{\textnormal{av}}_{\mu_n, j}[n]$.
\end{prop}

\begin{proof}
  Consider any $Q\in \Fix T_{\mu_n}^{(n)}$. With $\vect{z}_n \coloneqq (\vect{s}_n, a_n)$,
  where $\vect{s}_n$ and $a_n$ are the current state vector and action at iteration $n$,
  respectively, the reproducing property of the inner product of $\mathcal{H}$ yields
  $0 = \innerp{Q-T_{\mu_n}^{(n)}Q}{\varphi(\vect{z}_n)}_{\mathcal{H}}$, or via
  \eqref{Bellman.new.mu}:
  \begin{align}
    & 0 = \innerp{Q - g_n - \alpha \sum_{j=1}^{N_{\textnormal{av}}[n]}
      Q (\vect{s}^{\textnormal{av}}_j[n], \mu_n(\vect{s}^{\textnormal{av}}_j[n]) )
      \psi_j^{(n)} }{\varphi(\vect{z}_n)}_{\mathcal{H}} \notag \\
    \Leftrightarrow {}
    & {} g_n(\vect{z}_n) = Q(\vect{z}_n) - \alpha \sum_{j=1}^{N_{\textnormal{av}}[n]}
      Q (\vect{s}^{\textnormal{av}}_j[n], \mu_n(\vect{s}^{\textnormal{av}}_j[n]) )
      \psi_j^{(n)}(\vect{z}_n) \,. \label{Hyperplane.1}
  \end{align}


  Vectors $\{ \psi_j^{(n)} \}_{j=1}^{N_{\textnormal{av}}[n]}$ can be designed as follows:
  $\psi_j^{(n)} \in \{ \psi\in\mathcal{H} \given \innerp{\psi}{
    \varphi(\vect{z}_n)}_{\mathcal{H}} = \psi(\vect{z}_n) = 1/N_{\textnormal{av}}[n]\}$,
  $\forall j\in N_{\textnormal{av}}[n]$. Notice that the set from where $\psi_j^{(n)}$ is chosen
  from is a nonempty hyperplane in $\mathcal{H}$. Under this choice, \eqref{Hyperplane.1}
  becomes:
  \begin{align}
    & g_n(\vect{z}_n) = Q(\vect{z}_n) - \alpha \tfrac{1}{N_{\textnormal{av}}[n]}
      \sum\nolimits_{j=1}^{N_{\textnormal{av}}[n]} Q (\vect{s}^{\textnormal{av}}_j[n],
      \mu_n(\vect{s}^{\textnormal{av}}_j[n]) ) \notag \\
    \Leftrightarrow {}
    & {} g_n(\vect{z}_n) = \innerp{Q}{ \varphi(\vect{z}_n) - \alpha
      \tfrac{1}{N_{\textnormal{av}}[n]} \sum\nolimits_{j=1}^{N_{\textnormal{av}}[n]}
      \varphi^{\textnormal{av}}_{\mu_n, j}[n] \,}_{\mathcal{H}}
      \,, \label{Hyperplane.2}
  \end{align}
  which completes the proof.
\end{proof}

The one-step loss function $g_n\in\mathcal{H}$ is chosen here such that
\begin{align}
  g_n(\vect{z}_n) = \tfrac{1}{M_{\textnormal{av}}} \sum\nolimits_{k=0}^{M_{\textnormal{av}}-1}
  \log_{10} \frac{\lvert y_{n-k} -
  \vectgr{\theta}_{n+1}^\intercal\vect{x}_{n-k}\rvert}{\norm{\vect{x}_{n-k}}_2}
  \,, \label{one.step.loss}
\end{align}
where $\vectgr{\theta}_{n+1}$ is provided by \eqref{LMP}. Recall by \eqref{state2} that the
right-hand-side of \eqref{one.step.loss} is nothing but $s_2^{(n+1)}$. There always exists
$g_n\in\mathcal{H}$ such that \eqref{one.step.loss} is satisfied, since one can choose any
$g_n$ from the nonempty hyperplane:
$\{ g\in \mathcal{H} \given \innerp{g}{\varphi(\vect{z}_n)}_{\mathcal{H}} = g(\vect{z}_n) =
s_2^{(n+1)}\}$.

Given the current estimate $Q_n$ of the Q-function, there are several ways to update to
$Q_{n+1}$ from the hyperplane $\mathscr{H}_n$. For example,
$Q_{n+1} \coloneqq P_{\mathscr{H}_n}( Q_n )$, where $P_{\mathscr{H}_n}(\cdot)$ stands for the
(metric) projection mapping onto $\mathscr{H}_n$~\cite{HB.PLC.book}. Or, via the minimum-norm
solution $Q_{n+1} \coloneqq P_{\mathscr{H}_n}( 0 )$. Nevertheless, to offer even a more
standard approach, the classical steepest-descent methodology on the quadratic loss
$\mathscr{L}_n(Q) \coloneqq (1/2) [ \innerp{Q}{h_n}_{\mathcal{H}} - g_n(\vect{z}_n) ]^2$:
\begin{align}
  Q_{n+1}
  & \coloneqq Q_n - \eta \nabla \mathscr{L}_n(Q_n) \notag \\
  & = Q_n - \eta [ \innerp{Q_n}{h_n}_{\mathcal{H}} - g_n(\vect{z}_n) ] h_n \,,\label{Q.update}
\end{align}
is provided here, where $\eta$ is the learning rate (step size).

\sloppy Although there are many ways to generate samples
$\{ \vect{s}_j^{\textnormal{av}}[n] \}_{j=1}^{ N_{\textnormal{av}}[n] }$, see for
example~\cite{Minh:ICASSP23}, a different approach than~\cite{Minh:ICASSP23} is followed
here. In short, past data are re-used, as in
$\vect{s}_j^{\textnormal{av}} [n] \coloneqq [\log_{10} |y_{n+1-j} -
\vectgr{\theta}_{n+1}^\intercal \vect{x}_{n+1-j}|, s_2^{(n)},
\log_{10}\norm{\vect{x}_{n+1-j}}_2, s_4^{(n)}]^{\intercal}$,
$j\in\{ 1, \ldots, N_{\textnormal{av}}[n]\}$, to capitalize on the fact that RVs
$(\vect{x}_n)_{n\in\IntegerP}$ are IID in \cref{sec:tests}.

To robustify the proposed API, (prioritized) experience replay~\cite{ExperienceReplay} is
utilized to allow re-use of past data. To this end, an experience-replay (ER) buffer is
constructed to comprise information
$\{ \vect{s}_{\nu}, a_{\nu}, g_{\nu}, \{ \vect{s}_j^{\textnormal{av}}[\nu]
\}_{j=1}^{N_{\textnormal{av}}[\nu]} \}$ which is collected at instances $\nu$ taken from
$\{1, \ldots, n\}$. Whenever experience replay is applied, data from the ER buffer are
utilized. In short, the following route is followed at each $n$:
$Q_n \to \eqref{Q.update} \to [\text{Re-use past data from the ER buffer}] \to \eqref{Q.update}
\to Q_{n+1}$. Details on how to select information for the ER buffer and to utilize that
information in the proposed API will be reported in the journal version of the paper.

A direct application of \eqref{Q.update} may lead to memory and computational complications,
since at each $n$, \eqref{Q.update} potentially adds new kernel functions into the
representation of $Q_{n+1}$ via $h_n$. This unpleasant phenomenon is fueled by the potential
infinite dimensionality of $\mathcal{H}$; see for example the case where the kernel of
$\mathcal{H}$ is the Gaussian~\cite{scholkopf2002learning}
$\kappa_{\textnormal{G}}( \vect{z}, \vect{z}^{\prime} )$,
$( \vect{z}, \vect{z}^{\prime} )\in \mathcal{H}^2$, as in \cref{sec:tests}. To address this
``curse of dimensionality,'' this work employs the methodology of RFF~\cite{RFF}. Avoiding most
of the details due to space limitations,
$\kappa_{\textnormal{G}}( \vect{z}, \vect{z}^{\prime} )$ is approximated by the following inner
product $\tilde{\varphi}(\vect{z})^{\intercal} \tilde{\varphi}(\vect{z}^{\prime})$, where the
Euclidean feature vector
\begin{align}
    \tilde{\varphi} (\vect{z}) \coloneqq (\tfrac{2}{D})^{1/2}
  [ \cos{(\vect{v}_1^{\intercal} \vect{z} + b_1)}, \ldots, \cos{\vect{(v}_D^{\intercal}
  \vect{z} + b_D)} ]^{\intercal} \,, \label{RFF.feature}
\end{align}
with $D\in\IntegerPP$ being a user-defined dimension, while $\{\vect{v}_i\}_{i=1}^D$ and
$\{b_i\}_{i=1}^D$ are RVs following the Gaussian and uniform distributions, respectively. The
feature mapping \eqref{RFF.feature} is used instead of $\varphi(\cdot)$ throughout this work to
transfer learning from the infinite dimensional $(\mathcal{H}, \kappa_{\textnormal{G}})$ to the
$D$-dimensional $\Real^D$. Mapping \eqref{RFF.feature} together with the low-complexity
iteration \eqref{Q.update} yield an API with bounded computational complexity.

\section{Numerical Tests}\label{sec:tests}

\sloppy \cref{algo} is tested against
\begin{enumerate*}[label=\textbf{(\roman*)}]

\item \eqref{LMP}, for the values $p\in \mathfrak{A}\coloneqq \{1,1.25,1.5,1.75,2\}$, which
  are kept fixed throughout all iterations,

\item \cite{vazquez2012}, which uses a combination of adaptive filters with
  different forgetting factors but with the same fixed $p$-norm,

\item the kernel-based TD(0)~\cite{kernelTD1}, equipped with RFF and experience replay, and

\item the kernel-based (K)LSPI~\cite{xu2007klspi}; see \Cref{fig:vs.LMP,fig:vs.TD.KLSPI}.

\end{enumerate*}
Tests were also run to examine the effect of several of \cref{algo}'s parameters on
performance; see \cref{fig:vs.params}. The metric of performance is the normalized deviation
from the desired $\vectgr{\theta}_*$; see the vertical axes in all figures. The Gaussian
kernel~\cite{scholkopf2002learning} was used, approximated by RFF as described in
\cref{sec:algo}. The dimension $L$ of $\vect{x}_n, \vectgr{\theta}_*$ in \eqref{LMP} is $100$,
with a learning rate $\rho = 10^{-3}$. Both $\vect{x}_n$ and $\vectgr{\theta}_*$ are generated
from the Gaussian distribution $\mathcal{N}(\vect{0}, \vect{I}_{L})$, with
$(\vect{x}_n)_{n\in\IntegerP}$ designed to be IID. Moreover, $M_{\textnormal{av}} = 300$ and
$\varpi = 0.3$ in \eqref{def.states}, and $\eta = 0.5$ in \eqref{Q.update}.

Two types of outliers were considered. First, $\alpha$-stable outliers, generated
by~\cite{miotto2016pylevy}. Parameters $\alpha_{\textnormal{stable}} = 1$,
$\beta_{\textnormal{stable}} = 0.5$, $\sigma_{\textnormal{stable}} = 1$ were used, which yield
a considerably heavy-tailed distribution for the outliers. Second, ``sparse'' outliers were
generated, with values taken from the interval $[-100, 100]$ via the uniform
distribution. Sparse outliers appear in $10\%$ percent of the data, whereas in the rest $90\%$
of the data, Gaussian noise with $\textnormal{SNR} = 30\textnormal{dB}$ appears. As it is
customary in adaptive filtering, system $\vectgr{\theta}_*$ is changed at time $20,000$ to test
the tracking ability of \cref{algo}. Each test is repeated independently for $100$ times, and
uniformly averaged curves are reported.

As it can be verified by \Cref{fig:vs.LMP,fig:vs.TD.KLSPI,fig:vs.params}, \cref{algo}
outperforms the competing methods. KLSPI~\cite{xu2007klspi} fails to provide fast convergence and
performance close to the levels of the rest of the methods. The kernel-based
TD(0)~\cite{kernelTD1} converges fast, but with a subpar performance with regards to that of
the proposed framework. More tests on several other scenarios, together with the results
of~\cite{Minh:ICASSP23}, will be reported in the journal version of the paper.

\begin{figure}[t]
  \centering
  \subfloat[$\alpha$-stable outliers]{ \includegraphics[ width
    = .24\textwidth]{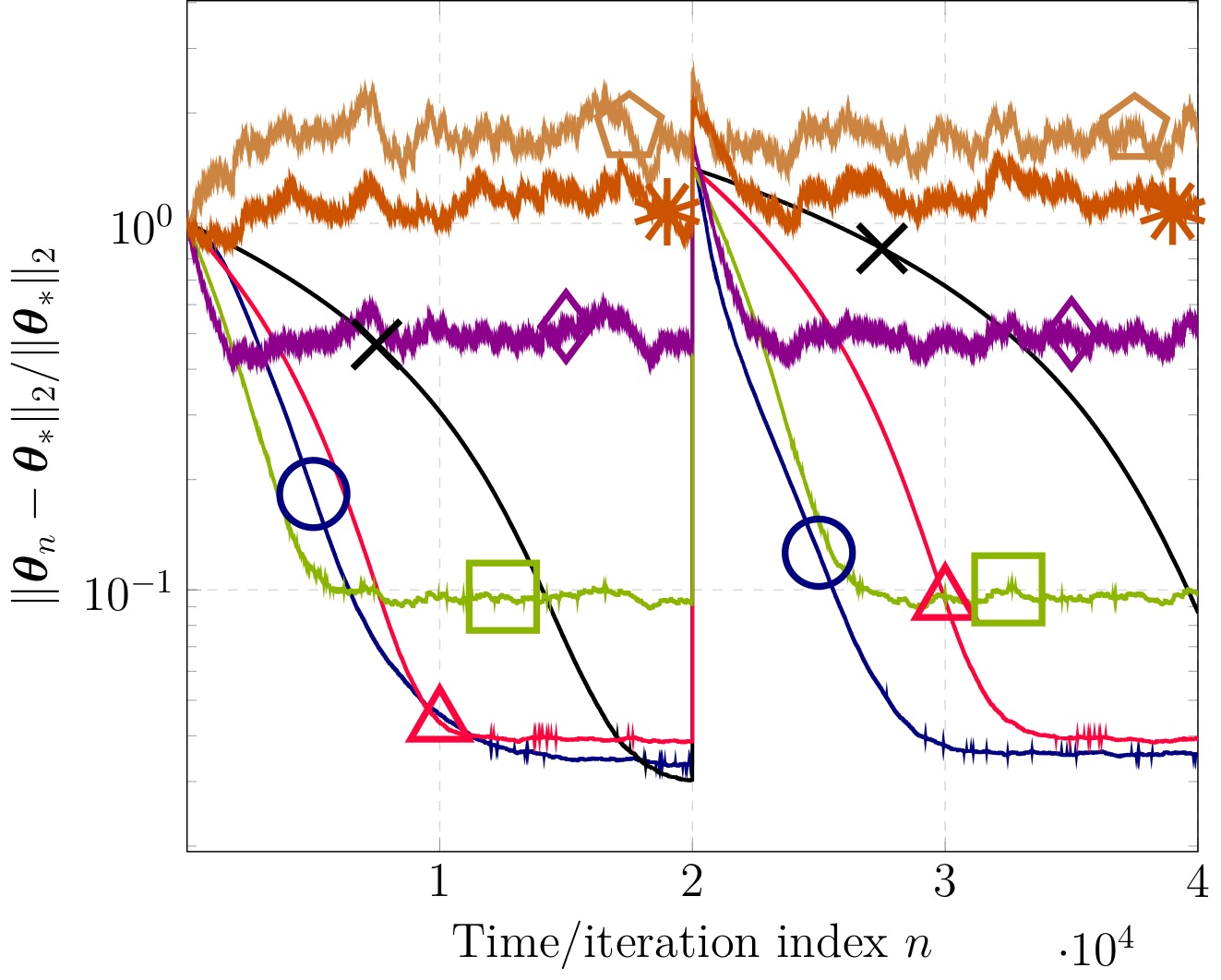}\label{1-1} }
  \subfloat[Sparse outliers]{ \includegraphics[ width =
    .24\textwidth]{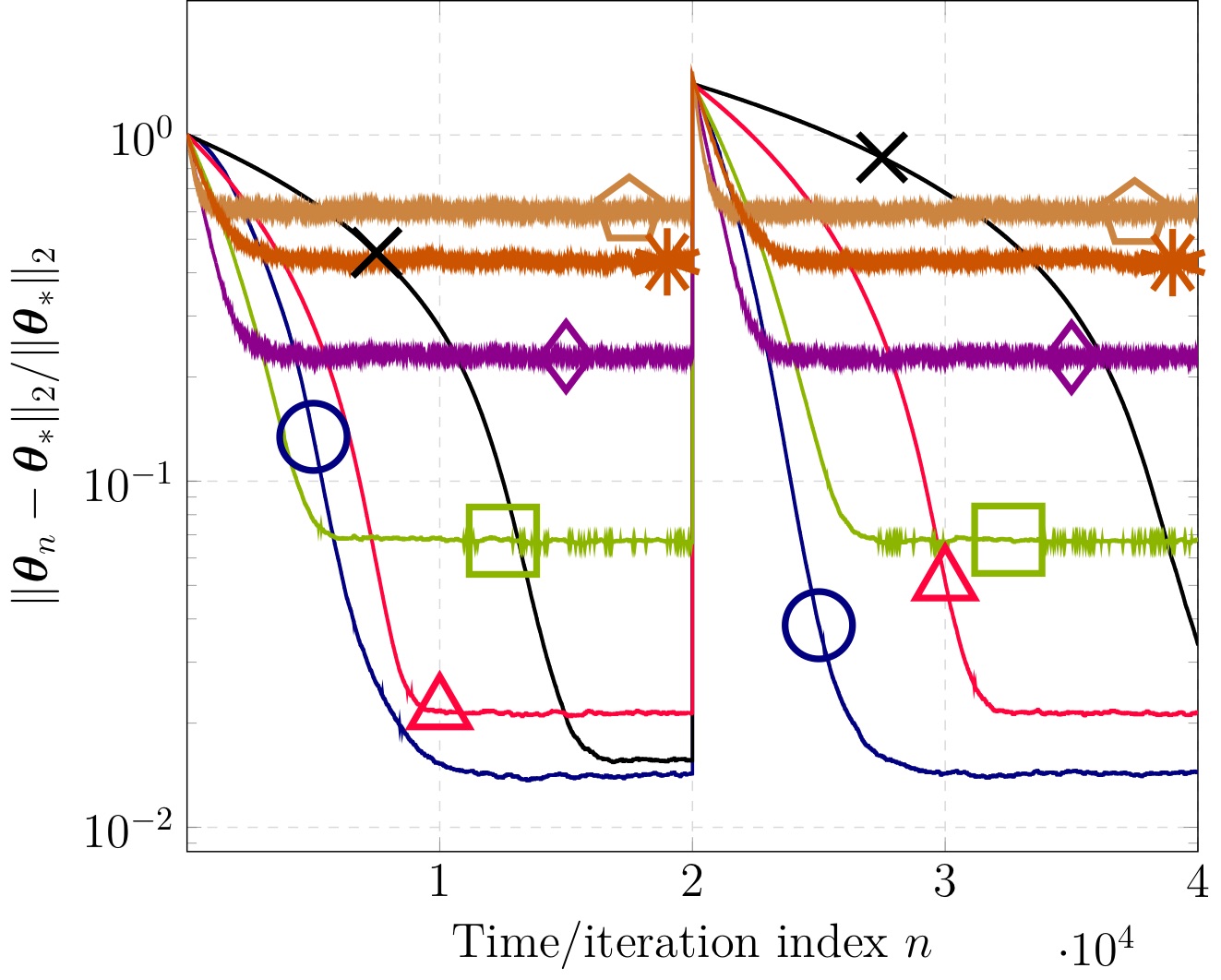}\label{1-2} }
  \caption{\protect\tikz[baseline = -0.5ex]{ \protect\node[mark size = 3pt, color = navy, line
      width = .5pt ] {\protect\pgfuseplotmark{o}}}: \cref{algo} w/
    $N_{\text{av}}=10, \alpha=0.75$. Marks \protect\tikz[baseline = -0.5ex]{ \protect\node[mark
      size = 3pt, color = black, line width = .5pt ] {\protect\pgfuseplotmark{x}}},
    \protect\tikz[baseline = -0.5ex]{ \protect\node[mark size = 3pt, color = americanrose, line
      width = .5pt ] {\protect\pgfuseplotmark{triangle}}}, \protect\tikz[baseline = -0.5ex]{
      \protect\node[mark size = 3pt, color = applegreen, line width = .5pt ]
      {\protect\pgfuseplotmark{square}}}, \protect\tikz[baseline = -0.5ex]{ \protect\node[mark
      size = 3pt, color = darkmagenta, line width = .5pt ] {\protect\pgfuseplotmark{diamond}}},
    \protect\tikz[baseline = -0.5ex]{ \protect\node[mark size = 3pt, color = peru, line width =
      .5pt ] {\protect\pgfuseplotmark{pentagon}}} correspond to \eqref{LMP} w/
    $p=1, 1.25, 1.5, 1.75, 2$, respectively. Mark \protect\tikz[baseline = -0.5ex]{
      \protect\node[mark size = 3pt, color = burntorange, line width = .5pt ]
      {\protect\pgfuseplotmark{10-pointed star}}} denotes an algorithm which randomly chooses
    $p$, $\forall n$.}\label{fig:vs.LMP}
\end{figure}

\begin{figure}[t]
  \centering \subfloat[$\alpha$-stable outliers]{ \includegraphics[width =
    .24\textwidth]{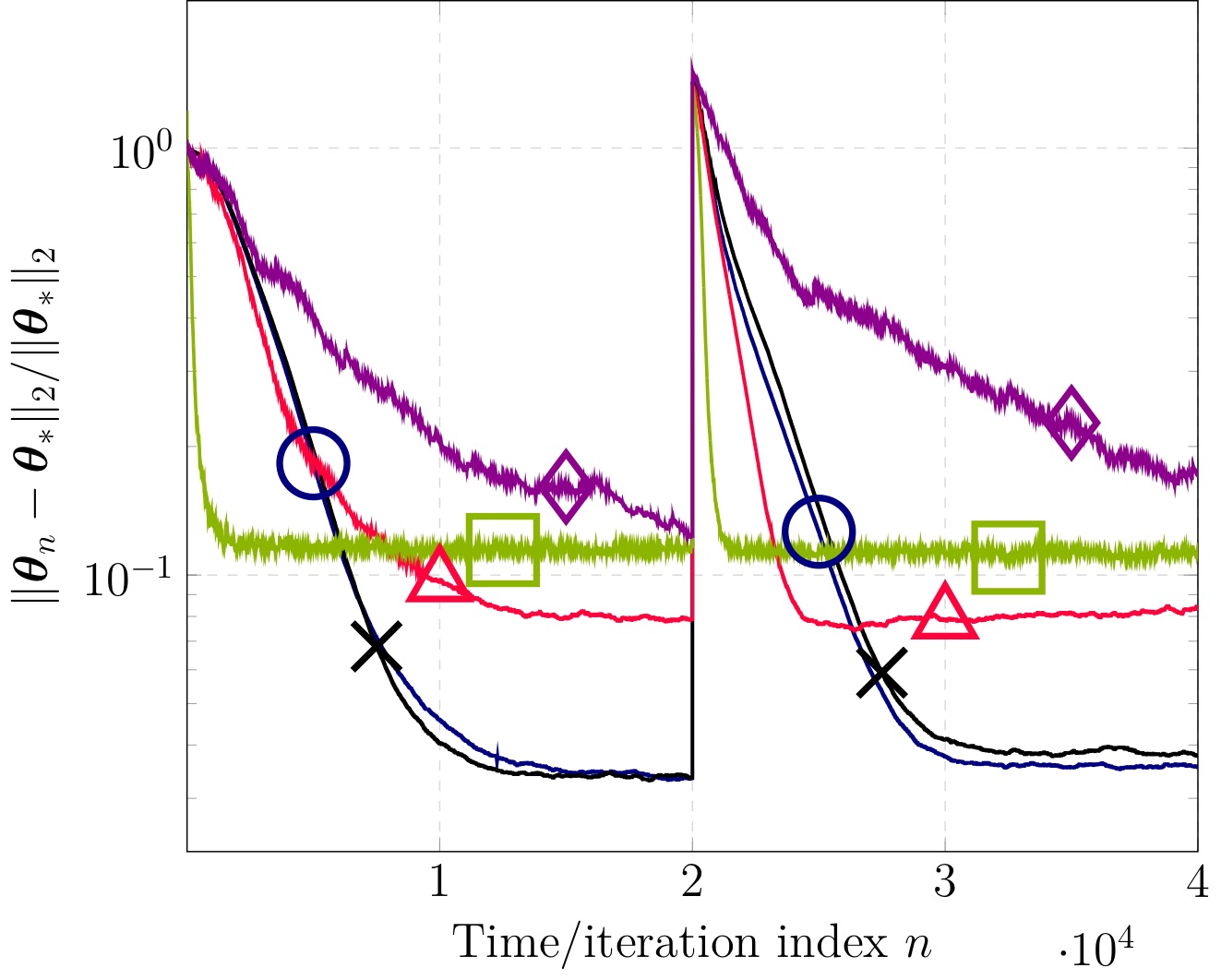}\label{3-1} }
  \subfloat[Sparse outliers]{\includegraphics[width = .24\textwidth]{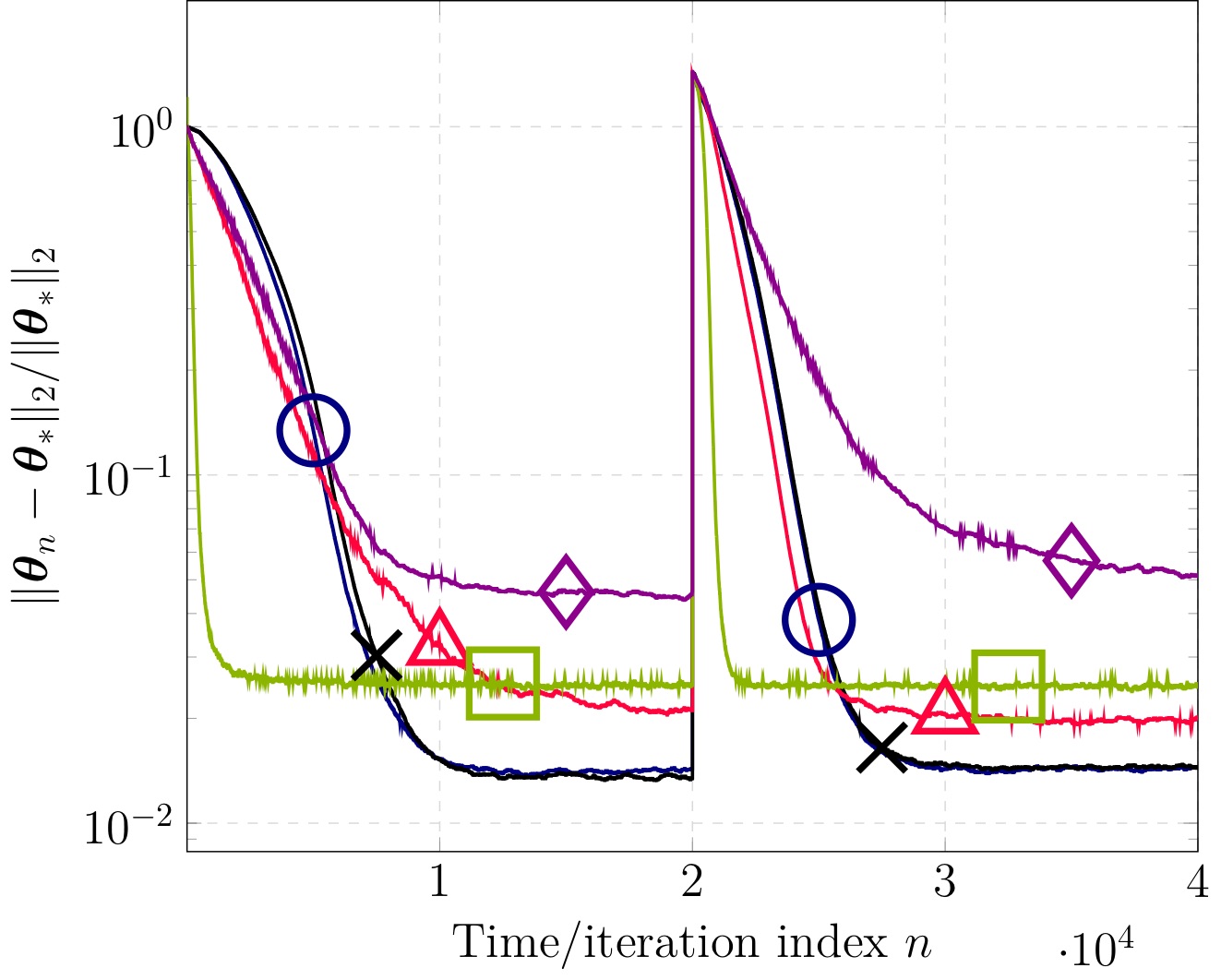}\label{3-2}}
  \caption{\protect\tikz[baseline = -0.5ex]{ \protect\node[mark size = 3pt, color = navy, line
      width = .5pt ] {\protect\pgfuseplotmark{o}}}: \cref{algo} w/
    $N_{\textnormal{av}}=10, \alpha=0.75$. \protect\tikz[baseline = -0.5ex]{ \protect\node[mark
      size = 3pt, color = black, line width = .5pt ] {\protect\pgfuseplotmark{x}}}: \cref{algo}
    w/ $N_{\textnormal{av}}=1, \alpha=0.75$. \protect\tikz[baseline = -0.5ex]{
      \protect\node[mark size = 3pt, color = americanrose, line width = .5pt ]
      {\protect\pgfuseplotmark{triangle}}}: Kernel-based TD(0) w/
    $\alpha = 0.9$~\cite{kernelTD1}. \protect\tikz[baseline = -0.5ex]{ \protect\node[mark size
      = 3pt, color = applegreen, line width = .5pt ] {\protect\pgfuseplotmark{square}}}:
    \cite{vazquez2012} w/ $p=1, \gamma_1 = 0.9, \gamma_2 =
    0.99$. \protect\tikz[baseline = -0.5ex]{ \protect\node[mark size = 3pt,
      color = darkmagenta, line width = .5pt ] {\protect\pgfuseplotmark{diamond}}}: KLSPI w/
    $\alpha=0.9$~\cite{xu2007klspi}}\label{fig:vs.TD.KLSPI}
\end{figure}

\begin{figure}[t]
  \centering \subfloat[$\alpha$-stable outliers]{\includegraphics[width = .24\textwidth]{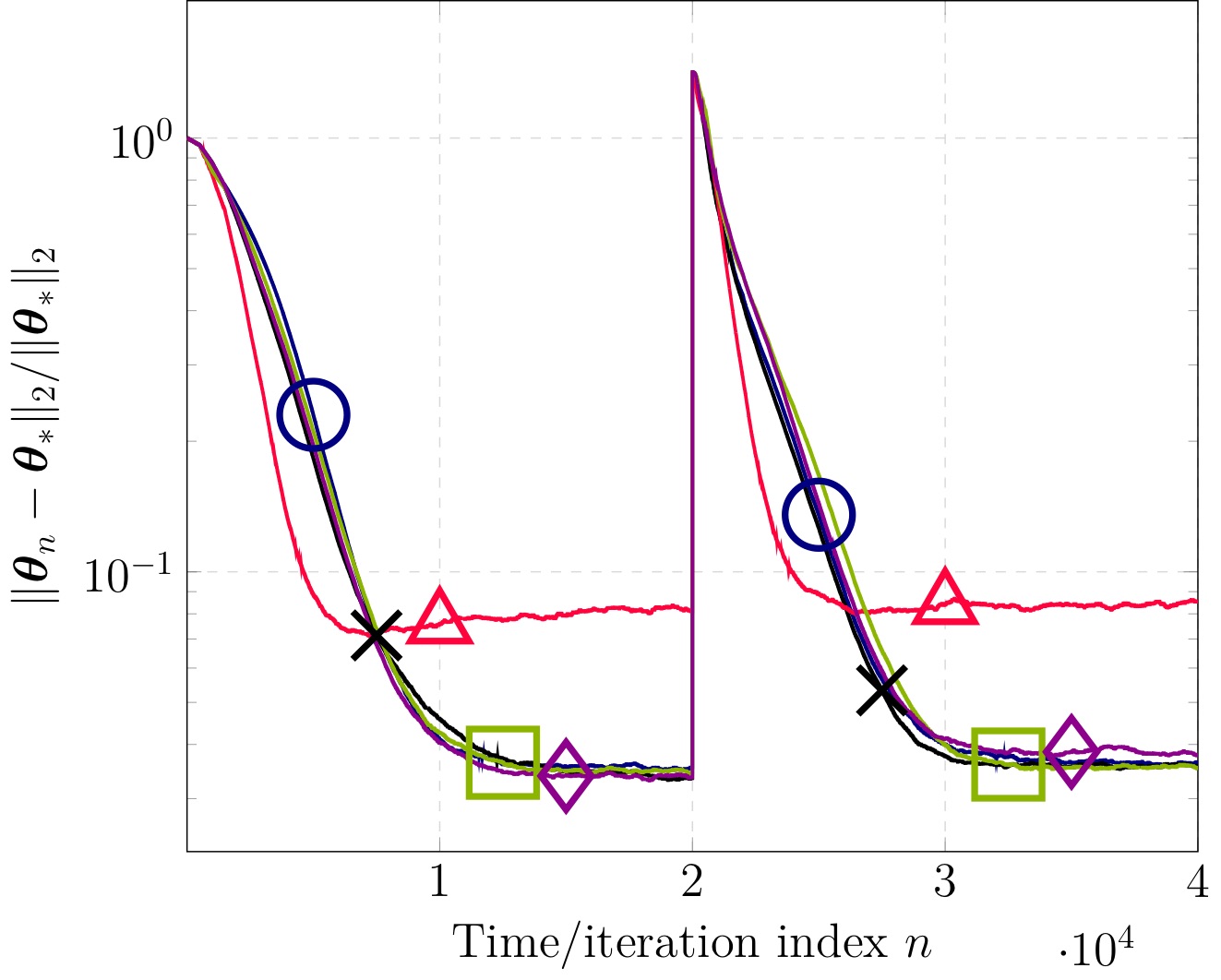}\label{2-1}}
  \subfloat[Sparse outliers]{\includegraphics[width = .24\textwidth]{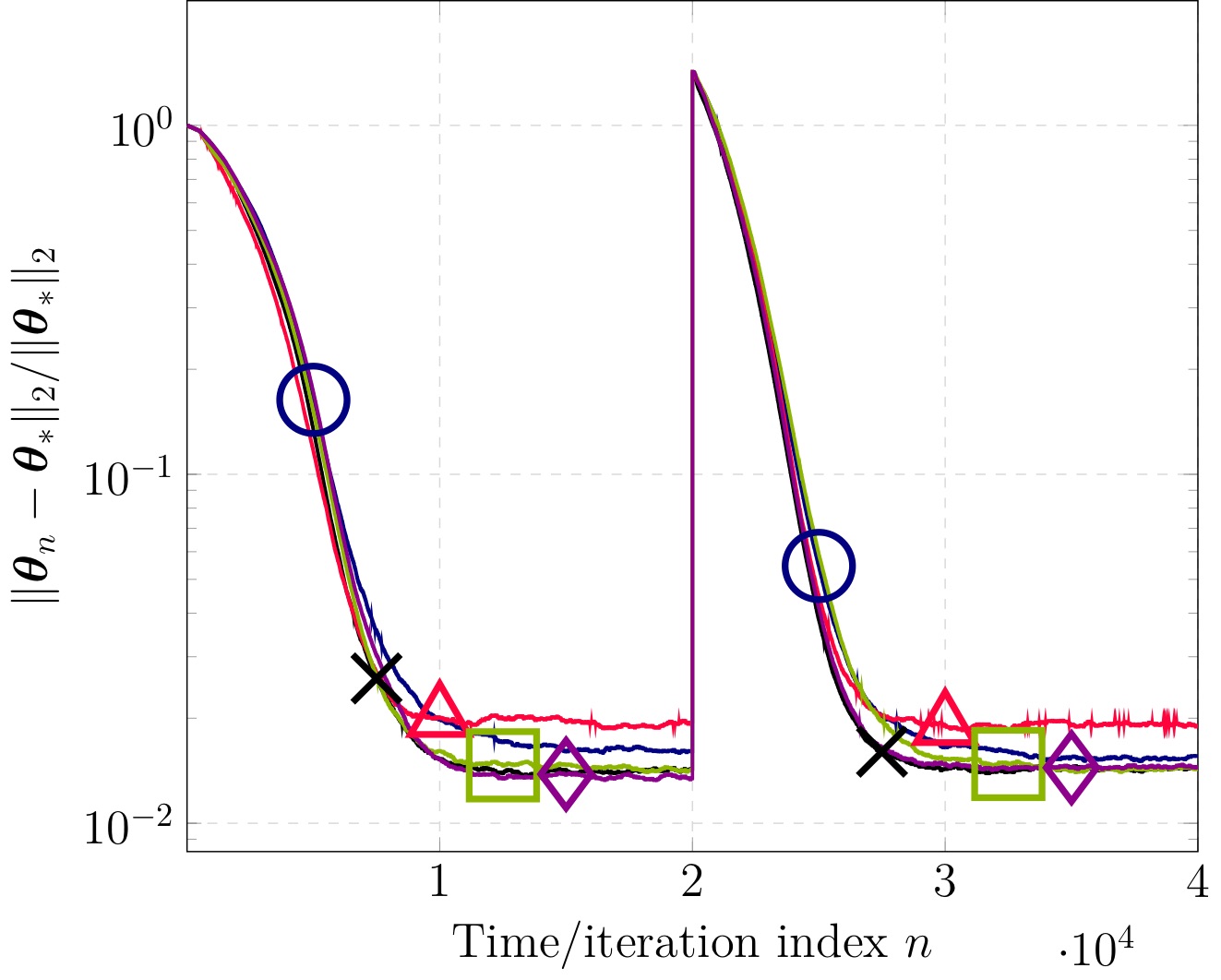}\label{2-2}}
  \caption{\cref{algo} w/ several parameters. \protect\tikz[baseline = -0.5ex]{ \protect\node[mark
      size = 3pt, color = navy, line width = .5pt ]
      {\protect\pgfuseplotmark{o}}}: $N_{\textnormal{av}}=10, \alpha=0.9$.
    \protect\tikz[baseline = -0.5ex]{ \protect\node[mark
      size = 3pt, color = black, line width = .5pt ]
      {\protect\pgfuseplotmark{x}}}: $N_{\textnormal{av}}=10, \alpha=0.75$.
    \protect\tikz[baseline = -0.5ex]{ \protect\node[mark
      size = 3pt, color = americanrose, line width = .5pt ]
      {\protect\pgfuseplotmark{triangle}}}: $\alpha=0$.
    \protect\tikz[baseline = -0.5ex]{ \protect\node[mark
      size = 3pt, color = applegreen, line width = .5pt ]
      {\protect\pgfuseplotmark{square}}}: $N_{\textnormal{av}}=1, \alpha=0.9$.
    \protect\tikz[baseline = -0.5ex]{ \protect\node[mark
      size = 3pt, color = darkmagenta, line width = .5pt ]
      {\protect\pgfuseplotmark{diamond}}}: $N_{\textnormal{av}}=1,
    \alpha=0.75$.}\label{fig:vs.params}
\end{figure}

\clearpage

\footnotesize
\bibliographystyle{ieeetr}
\bibliography{refs}

\end{document}